\documentclass[11pt,b5paper]{article}
\usepackage[margin=2.8cm]{geometry}

\usepackage{amsmath,amsthm,amssymb,amsfonts} 
\usepackage{graphicx}
\usepackage{setspace}
\usepackage[dvipsnames]{xcolor}
\usepackage{enumitem}
\usepackage{authblk}
\usepackage{hyperref}
\usepackage{cleveref} 
\usepackage{microtype}
\usepackage{pgfplots}
\usepackage{booktabs} 
\usepackage{subcaption} 
\usepackage{url}

\usepackage[utopia,cal=cmcal]{mathdesign} 
\usepackage{mathtools}
\usepackage{thmtools}
\usepackage{thm-restate}
\usepackage[font=footnotesize,labelfont=sc]{caption}
\usepackage{xfrac}

\usepackage{tikz}
\usetikzlibrary{shapes.geometric, arrows, positioning, decorations.markings, decorations.pathmorphing, calc}
\tikzstyle{startstop} = [rectangle, semithick, minimum width=0.5cm, minimum height=0.5cm,text centered, draw=black]
\tikzstyle{chance} = [circle, semithick, minimum width=0.5cm, minimum height=0.5cm, text centered, draw=black]
\tikzstyle{decision} = [rectangle, semithick, minimum width=0.5cm, minimum height=0.5cm, text centered, draw=black]
\tikzstyle{end} = [minimum width=0.5cm, minimum height=0.5cm, text centered, draw=white]
\tikzstyle{arrow} = [thick,->,>=latex]
\usepackage{tikz-cd}
\tikzset{
	commutative diagrams/.cd, 
	arrow style=tikz, 
	diagrams={>=stealth}
}

\usepackage[round,authoryear]{natbib}
\newcommand{\betterthan}{\succ} 
 
\newcommand{\atleastasgoodas}{\succsim}

\newcommand{\ooutcome}{o}

\newif\ifanon
\anonfalse

\title{Will artificial agents pursue power by default?}
\ifanon
\author{}
\date{}

\else
\author{Christian Tarsney}
\date{June 2025}
\fi

\begin{document}
	\maketitle
	
	\begin{abstract}
		Researchers worried about catastrophic risks from advanced AI have argued that we should expect sufficiently capable AI agents to pursue power over humanity because power is a \emph{convergent instrumental goal}, something that is useful for a wide range of final goals. Others have recently expressed skepticism of these claims. This paper aims to formalize the concepts of instrumental convergence and power-seeking in an abstract, decision-theoretic framework, and to assess the claim that power is a convergent instrumental goal. I conclude that this claim contains at least an element of truth, but might turn out to have limited predictive utility, since an agent’s options cannot always be ranked in terms of power in the absence of substantive information about the agent’s final goals. However, the fact of instrumental convergence is more predictive for agents who have a good shot at attaining absolute or near-absolute power.

		\vspace{\baselineskip}
		
		
	\end{abstract}

	\section{Introduction}
	
	With the rapid progress in artificial intelligence (AI) capabilities over the past several years, worries about catastrophic risks from advanced AI have become increasingly prominent.\footnote{For recent surveys of these concerns, see \cite{hendrycks2023overview} and \cite{bales2024artificial}.} Many of these worries (though by no means all) focus on artificial \emph{agents}: very roughly, AI systems that have goals in the external world, beliefs about the world, and behave in whatever ways they believe will best achieve their goals.

	It's a matter of debate, and not at all obvious, whether any AI system with sufficient cognitive capabilities \emph{must} be agentic in this sense. 
	But it does seem likely that, as AI capabilities advance, we will \emph{choose} to create increasingly agentic systems. For instance, present-day large language models (LLMs) aren't particularly agentic on their own, but can be embedded within larger ``language agent'' systems that behave agentically: a human user supplies an objective in natural language, the LLM develops a plan to achieve that objective, and that plan is automatically implemented, with updates based on feedback from the environment.\footnote{See for instance \cite{park2023generative}, \cite{wang2023voyager}. For general discussion of language agents and their implications for AI catastrophic risk, see \cite{goldstein2023language}.} AI systems that can carry out complex plans autonomously can substitute for human labor, and perhaps eventually perform many tasks better than human workers, giving them potentially enormous economic value.

	The \emph{alignment problem} is, roughly, the problem of making sure that artificial agents want what we want them to want \citep{gabriel2020artificial,ngo2023alignment}. The more capable and widely deployed artificial agents become, the more the world will be influenced by their activities, and the more important it will be that they are pursuing the right goals.
	
	Many researchers believe that aligning AI agents with human-level or greater capabilities will be extremely difficult, and that the creation of such agents therefore carries catastrophic risks.\footnote{See for instance \cite{bostrom2014superintelligence}, especially the discussion of ``motivation selection methods'' in Ch.\ 9; \cite{carlsmith2022powerseeking}, especially \S4; \citet[\S 5]{hendrycks2023overview}; \cite{ngo2023alignment}; \citet[\S 2]{bales2024artificial}; \citeauthor{cappelenFCsurvival} (forthcoming, \S3).} Two claims that often underlie these worries are that (i) there are some \emph{default tendencies} for artificial agents to want things that we don't want them to want, and (ii) it's at least difficult to ensure before deployment that we've eliminated all those undesirable default tendencies from an AI agent.

	My focus in this paper will be on the first of those claims---more specifically, on the idea that AI agents will end up with undesirable goals by default as a result of \emph{instrumental convergence}. Roughly, instrumental convergence is the idea that certain goals are \emph{convergent} in the sense that they are useful for achieving a very wide range of further goals. So, even if we are totally ignorant about an agent's \emph{final} goals (the ends that the agent values for their own sakes), if we can predict that the agent will be minimally instrumental rational, then we can make some non-trivial predictions about its \emph{instrumental} goals---namely, that its instrumental goals will include many or all of these convergent goals. 
	The idea of instrumental convergence has been advocated most prominently by \cite{omohundro2008basic} and \cite{bostrom2012superintelligent}. As Bostrom defines it, a convergent instrumental goal is something the attainment of which ``would increase the chances of the agent's goal being realized for a wide range of final goals and a wide range of situations, implying that these instrumental values are likely to be pursued by many intelligent agents'' \citep[p.\ 76]{bostrom2012superintelligent}.
	
	Omohundro, Bostrom, and others have proposed a number of goals as being convergent instrumental goals in this sense. Perhaps the most obvious is \emph{self-preservation}: You can't fetch the coffee if you're dead, you can't make paperclips if you're dead, etc \citep[Ch.\ 5]{russell2019human}, so almost no matter what your final goals are, it's useful to stay alive. Another commonly suggested candidate is \emph{final goal preservation}. If a powerful agent goes from wanting to maximize the number of paperclips in the world to wanting to minimize the number of paperclips, the result is likely to be fewer paperclips. So an instrumentally rational paperclip maximizer will, by default, want to prevent its final goals from being tampered with.\footnote{The ``paperclip maximizer'' is a stock example of an AI agent with arbitrary and potentially dangerous final goals, due to \cite{bostrom2003ethical}.} Other proposed examples of convergent instrumental goals include resources, information, cognitive self-enhancement, and keeping one's options open, among others. 
	
	Many if not all of these proposed convergent instrumental goals can be subsumed under the general umbrella of \emph{power-seeking}. For instance, self-preservation is instrumentally valuable because if you're alive, you have at least some power to pursue your goals, and if you're dead, you don't. Likewise, resources, information, and cognitive enhancement are all clearly ways of increasing an agent's power to achieve its goals.\footnote{I think one could argue that \emph{all} plausible convergent instrumental goals are different forms of or contributors to power, though I won't argue for that claim here and nothing I'm going to say will depend on it.}

	Though the concept of power plays little explicit role in \cite{omohundro2008basic} or \cite{bostrom2012superintelligent}, many subsequent researchers have identified it as a potential convergent instrumental goal.\footnote{\textcolor{black}{The idea that power is a convergent instrumental goal is discussed sympathetically, for example, by \cite{turner2021optimal}, \citet[esp.\ \S 4.2]{carlsmith2022powerseeking}, \citet[\S 5.3]{hendrycks2023overview}, \cite[\S 5]{dung2024argument}, and \citeauthor{cappelenFCsurvival} (forthcoming, \S 3).}} And the idea of power-seeking has figured centrally in recent worries about catastrophic risks from AI.\footnote{\textcolor{black}{See for instance \cite{carlsmith2022powerseeking,carlsmithFCexistential}, \cite{hendrycks2023overview}, \cite{dung2024argument}, \cite{bales2024artificial}, \cite{cappelenFCsurvival}, and \citeauthor{ngoFCdeceit} (forthcoming).}} For instance, Joseph Carlsmith writes: ``[P]ower is extremely useful to accomplishing objectives---indeed, it is so almost by definition. So to the extent that an agent is engaging in unintended behavior in pursuit of problematic objectives, it will generally have incentives, other things equal, to gain and maintain forms of power in the process'' \citep[p.\ 18]{carlsmith2022powerseeking}.

	To claim that power is a convergent instrumental goal is to claim that there is, in a certain sense, a ``default tendency'' for agents with most possible final goals to engage in power-seeking behavior, at least if they're instrumentally rational. If that's true, so what? How does it support the claim that creating highly capable AI agents carries catastrophic risks? 
	
	The argument from instrumentally convergent power-seeking to catastrophic risk involves, I think, two further significant premises. 
	First, if agents do have a default tendency to pursue some goal $g$, then it's at least hard to be very confident, before deploying a new artificial agent for the first time, that we've trained that tendency out of it. This claim relies on two subpremises. The first is what \cite{bostrom2012superintelligent} calls the \emph{orthogonality thesis}, which says roughly that there's no intrinsic relationship between intelligence and final goals: the fact that we've trained a very intelligent, very capable AI system, that's very good at means-ends reasoning, doesn't guarantee that it will have any particular final goals, like human welfare, the objective moral good, or an intrinsic aversion to power. 
	The second subpremise is what we might call \emph{alignment difficulty}. As the name suggests, this is the claim that aligning an artificial agent with a desired set of values is a hard problem---or at least, it's hard to be confident that you've succeeded before you actually deploy the agent in the real world and give it an opportunity to pursue its goals. 
	These claims together at least suggest that, if there is a default tendency for agents with most possible final goals to pursue instrumental goal $g$, then we shouldn't be \emph{extremely} confident pre-deployment that the agent we've created won't also pursue $g$, even (or perhaps especially) if that agent is extremely intelligent and capable.
	
	The second high-level premise is that an AI agent that has superhuman capabilities and instrumentally values power would pose catastrophic risks to humanity. There are various reasons why you might think that. For instance, you might think that if an AI agent has superhuman capabilities, it's likely to be \emph{successful} in its pursuit of power, that its successful pursuit of power would result in humanity being \emph{dis}empowered, and that this outcome would be intrinsically catastrophic.\footnote{For a nuanced discussion of this last claim, see \citeauthor{balesFCtakeover} (forthcoming-a).} Or you might think that, if a superhuman AI agent successfully attains great power, it will likely also succeed in achieving its final goals, and it's likely that those \emph{final} goals will be catastrophic for humanity---for instance, converting all matter in the solar system into paperclips. Or, finally, you might think that the best \emph{means} for a superintelligent AI agent to pursue power will involve actions that are catastrophic for humanity, like wiping us out or enslaving us to ensure that we can't threaten its power in future.
	
	Instrumentally convergent power-seeking, the difficulty of reliably eliminating default tendencies in AI agents pre-deployment, and the potentially catastrophic harms from superhuman AI agents seeking power together suggest that the creation of such agents is catastrophically risky.
	
	Recently, though, some philosophers have pushed back on this story, and in particular on the initial premise that power is a convergent instrumental goal that agents have a default tendency to pursue. In particular, \cite{gallow2024instrumental} argues that there are some genuinely convergent instrumental goals, but doesn't think that power-seeking is among them. And \cite{thorstad2024power} examines some formal arguments for instrumentally convergent power-seeking from the reinforcement learning literature, and argues that they don't show what they purport to show, and in particular don't show anything that strongly supports worries about catastrophic AI risk.\footnote{For a skeptical perspective on instrumental convergence arguments more generally, see \cite{sharadinFCpromotionalism}.}
	
	Gallow's and Thorstad's arguments are both primarily negative: They don't claim to provide comprehensive positive arguments \emph{against} instrumentally convergent power-seeking. Rather, they find a lack of support \emph{for} instrumentally convergent power-seeking, or at any rate for a version of that claim that could play a role in a compelling argument for catastrophic risks from AI. What their papers show is that, a bit surprisingly, there is not yet a satisfactory general formalization of the claim that power is a convergent instrumental goal, or a formal argument for that general claim.
	
	There are, as Thorstad discusses at length, a number of ``power-seeking theorems'' in the reinforcement literature \citep{benson2016formalizing,turner2021optimal,turner2022parametrically,krakovna2023powerseeking}. But while these theorems make valuable contributions, they have at least two limitations: First, they make fairly specific assumptions about the agent's final goals---in particular the assumption, characteristic of reinforcement learning, that the agent is maximizing or has been trained to maximize a discounted sum of time-separable rewards.
	Second, they make substantial assumptions about the decision situations the agent faces. For instance, \cite{benson2016formalizing} assume that the agent operates in an environment divided into discrete regions, and that its choices concern the allocation of resources across regions, with particular rules for how the state of each region evolves, how the allocation of resources at one time-step generates new resources at the next time step, and so on. Results that depend on substantial assumptions about an agent's final goals or about the environment can provide only limited support for claim that power or some other goal is instrumentally convergent, in Bostrom's sense of ``[increasing] the chances of the agent's goals being realized for a wide range of final goals and a wide range of situations''.

	Nevertheless, it 
	seems to me that there's something importantly right about the idea that power is a convergent instrumental goal. My aim in this paper is to articulate that important kernel of truth, in a way that's more precise than the informal presentations of the idea in \cite{omohundro2008basic}, \cite{bostrom2012superintelligent}, and elsewhere, and more general than the power-seeking theorems in the reinforcement learning literature.\footnote{In this existing literature, the closest precedent for my positive arguments is \cite{turner2021optimal}. In particular, there are close conceptual parallels between our respective formalizations of \textit{power} and \textit{instrumental convergence}. Our definitions and results differ in a number of ways, but the most important difference is that Turner et al.\ work in the standard reinforcement learning setting of Markov decision processes (where an agent follows an infinite trajectory through a finite set of states, aiming to maximize a discounted sum of state-based rewards), while I work in the decision-theoretic setting of finite decision trees with an arbitrary utility function over final outcomes (to be explained in the next section).} On the other hand, I will also point out some potential limits on the predictive utility of this truth, that is, limits on how much it allows us to predict about the behavior of future artificial agents.

	Here's the plan: After some formal setup (\S \ref{s:setup}), I'll formalize the notion of instrumental convergence---or rather, three notions of instrumental convergence, with different degrees of strength---in an abstract decision-theoretic framework (\S \ref{s:instrumentalConvergence}). In \S \ref{s:power}, I'll propose several few ways of formalizing the notion of power in that same framework, and show that they each (to different degrees) constitute convergent instrumental goals. 
	In \S \ref{s:incompletenessResult}, I'll prove a negative result: Neither power nor any other convergent instrumental goal provides a \emph{complete} ordering of the options an agent might face. In many possible situations, the agent's options will be unranked in terms of any notion of power that could constitute a convergent instrumental goal, and so the fact of instrumentally convergent power-seeking doesn't help us predict the agent's behavior in those situations. In \S \ref{s:absolutepower}, however, I'll argue that instrumental convergence results are more likely to be predictive of the behavior of future artificial agents insofar as we expect those agents to have opportunities to acquire \emph{absolute} or \emph{near-absolute} power. \S \ref{s:conclusion} will conclude by surveying a few of the important questions about instrumentally convergent power-seeking left open by the preceding discussion.

	To be clear at the outset, my aim in this paper will \emph{not} be to defend the whole argument that power-seeking tendencies make advanced AI agents catastrophically risky. I'm not sure how much weight I ultimately give to that argument---in particular, it seems unclear at the moment whether aligning artificial agents in a way that reliably eliminates undesirable default tendencies will turn out to be a hard problem.\footnote{On this question, see for instance  \cite{goldstein2023language}, who argue that language agents are easier to reliably align and therefore less risky than reinforcement learning agents.} 
	My aim is just to partially vindicate one premise of that argument.
	
	\section{Setup}
	\label{s:setup}
	
	We'll work in the framework of sequential decision theory. That is, we assume an agent facing a \emph{sequential choice situation}, represented by a decision tree. The tree consists of \emph{choice nodes}, \emph{chance nodes}, \emph{(final) outcomes}, and \emph{branches}. The agent follows a path through the tree, from an initial node, until eventually reaching an outcome. At choice nodes, the path she takes depends on her choice. At chance nodes, the path she takes depends on a chance event, with each branch having a predetermined probability. Figure \ref{fig:decisionTree} gives a simple illustration of a decision tree.

	We will assume a set $\mathbb{O}$ of possible outcomes that is finite but has cardinality $\geq 3$. This defines a set $\mathcal{T}_\mathbb{O}$ of all finite decision trees (that is, decision trees with a finite number of nodes) terminating in outcomes from $\mathbb{O}$. A \emph{strategy} for a sequential choice situation is a specification, for each choice node in the decision tree, of either an action or a probability distribution over actions. A decision tree $t$ and a strategy $s$ for that tree together define a \emph{lottery} $L(t,s)$, the probability distribution over final outcomes for an agent following strategy $s$ from the initial node of tree $t$. We can also associate each decision tree $t$ with a \emph{set} of lotteries $\mathcal{L}(t)$, namely, the set of lotteries associated with all possible strategies for tree $t$. These are, intuitively, the lotteries \emph{available} to an agent at the initial node of $t$, that she can choose from by adopting a particular strategy. Each node in a decision tree defines a \emph{subtree}, the part of the overall decision tree accessible from that node.
	
	\begin{figure}
		\centering
		
		\begin{tikzpicture}
			
			\node (start) [startstop] {}; 
			
			\node[align=center] (C2) [decision, right of=start, xshift=2cm, yshift=-1cm] {};
			
			\node[align=center] (C3) [chance, right of=start, xshift=2cm, yshift=1cm] {};
			
			\node (A) [end, right of=C3, xshift=2cm, yshift=.5cm] {$o_1$};
			
			\node (D) [end, right of=C3, xshift=2cm, yshift=-.5cm] {$o_2$};
			
			\node (A-) [end, right of=C2, xshift=2cm, yshift=.5cm] {$o_3$};
			
			\node (B) [end, right of=C2, xshift=2cm, yshift=-.5cm] {$o_4$};
			
			\draw	[thick, black] (start) node[xshift=0.8cm, yshift=-0.5cm] {} -- (C2);
			
			\draw	[thick, black] (start) node[xshift=0.8cm, yshift=-0.5cm] {} -- (C3);
			
			\draw [arrow, thick, black] (C3) node[xshift=1.3cm, yshift=0.5cm] {\tiny $1/4$} -- (A);
			
			\draw [arrow, thick, black] (C3) node[xshift=1.3cm, yshift=-0.5cm] {\tiny $3/4$} -- (D);
			
			\draw [arrow, thick, black] (C2) node[xshift=1cm, yshift=0.25cm] {} -- (A-);
			
			\draw [arrow, thick, black] (C2) node[xshift=1cm, yshift=0.25cm] {} -- (B);
			
		\end{tikzpicture}		
		\caption{An example of a decision tree. Square nodes represent choices, circular nodes represent chance events, and numbers along branches represent probabilities.}
		\label{fig:decisionTree}
	\end{figure}
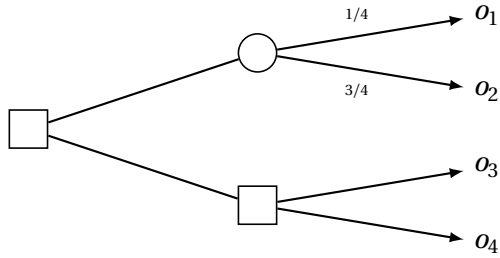
	
	Finally, we will assume that our agent is an \emph{expected utility maximizer}. This means that there is some utility function $u : \mathbb{O} \to \mathbb{R}$ such that, in any decision tree $t \in \mathcal{T}_\mathbb{O}$, she will follow a strategy $s$ such that the associated lottery $L(t,s)$ gives her at least as great an expectation of $u$ (expected utility) as any of the other available lotteries in $\mathcal{L}(t)$. This is undeniably a substantive assumption. However, it is not an assumption about the content of the agent's \emph{final goals}, which are represented by her ranking of final outcomes---expected utility maximization is compatible with any ranking of final outcomes. Rather, it's an assumption about what it means for an agent to be \emph{instrumentally rational}, and the standard such assumption in the relevant machine learning and AI safety literatures.\footnote{One potential substantive justification for this assumption is the idea that any sufficiently optimized agent can be represented as an expected-utility maximizer, because non-expected utility maximizers are subject to sure losses of a sort that will be penalized and thus eliminated by optimization processes \citep{yudkowky2015sufficiently}. I'm skeptical of these arguments, though. For recent critical appraisals, see \cite{thornley2023coherence}, \citeauthor{balesFCavoid} (forthcoming).} Since we are assuming that the set of possible outcomes $\mathbb{O}$ is finite, utilities are necessarily bounded, and without loss of generality we can think of them as confined to the unit interval $[0,1]$.

	\textcolor{black}{In a slight abuse of notation, we will let $u(l)$ denote the expected $u$-utility of lottery $l$ and $u(t)$ denote the maximum expected $u$-utility of any lottery available in decision tree $t$ (i.e., $\max_{l \in \mathcal{L}(t)}u(l)$). $u(t)$ represents the expected utility for an agent with utility function $u$ of reaching the initial node of tree $t$, given the assumption that its future behavior will continue to maximize the expectation of $u$.}
	
	Expected utility maximizers decide what to do at a given choice node in a decision tree based on the sets of lotteries available at subsequent nodes. That is, if such an agent is at a choice node with $n$ options, leading to subtrees $t_1$ to $t_n$, she will choose the subtree $t_i$ whose set of available lotteries $\mathcal{L}(t_i)$ contains the lottery with greatest expected utility (or will choose \emph{one} such subtree, in the case of ties). In other words, she chooses based on a fixed ranking of subtrees, without regard to the particular set of alternatives she faces at a given choice node, the larger structure of the decision tree in which she finds herself, or her own past choices.
	
	\section{Instrumental convergence}
	\label{s:instrumentalConvergence}
	
	We now have the resources to formalize the notions of \emph{instrumental convergence} and \emph{power}.
	
	Starting with instrumental convergence: The informal idea of instrumental convergence is that $g$ is a convergent instrumental goal if \emph{most} final goals make it instrumentally rational to pursue $g$. We might paraphrase this by saying that an instrumentally rational agent with \emph{randomly generated} final goals will be more likely than not to pursue $g$. 
	An expected utility maximizer's final goals are described by its utility function over outcomes. Thus, we can formalize the idea of randomly generated final goals in terms of a randomly generated utility function.
	
	But what do we mean by \emph{randomly generated}? In the context of characterizing instrumental convergence, the point is to capture the idea of \emph{ignorance} about an agent's final goals. This seems to license three assumptions: First, there should be no \textit{a priori} partiality between different outcomes. A randomly generated utility function is just as likely to assign a given utility to one outcome as to another; and more generally, any assignment of utilities to outcomes is exactly as likely as any of its possible permutations (swapping utilities between outcomes, while still assigning any given utility to the same \emph{number} of outcomes). Let's say that a probability distribution over utility functions is \emph{symmetric} if it satisfies this permutation invariance condition.\footnote{\cite{turner2021optimal} also make use of permutation invariance to operationalize claims about what is true for ``most'' final goals.}
	
	Second, no utility function can be ruled out \emph{a priori}: If we think of utility assignments to a set of $n$ outcomes as points in the $n$-dimensional hypercube $[0,1]^\mathbb{O}$, the probability distribution from which a randomly generated utility function is drawn should assign non-zero probability density to every point in that region. Equivalently, any measurable subset of $[0,1]^\mathbb{O}$ with non-zero volume (formally, non-zero Lebesgue measure) should be assigned non-zero probability. Let's say that a probability distribution over utility functions is \emph{regular} if it satisfies this condition.
	
	Finally, we will assume that the set of \emph{uniform} utility functions, which assign the same utility to every outcome, has probability zero. This reflects the idea that a system with a uniform utility function has, in effect, no final goals, and seems not to count as a genuine ``agent''. Insofar as we are interested in predicting the behavior of agents, then, we can assume non-uniformity.\footnote{We could go further and assume that uniform utility functions don't just have probability zero but are \emph{impossible}. But the probability-zero assumption will be sufficient for our purposes, and is formally simpler.} Let's say that a probability distribution over utility functions is \emph{almost surely non-uniform} if it satisfies this condition.
	
	Beyond these three conditions, we will assume nothing else about the distribution from which randomly generated utility functions are drawn---we will not, for instance, assume that they are drawn from a uniform distribution on $[0,1]^\mathbb{O}$ or that the utilities assigned to different outcomes are probabilistically independent. Our aim is to capture the idea of ignorance about an agent's final goals (beyond the assumption that the agent will have \emph{some} non-trivial goals), and so we avoid making any formal assumptions that are not clearly justified by this informal idea of ignorance.
	
	Putting the preceding conditions together:

	\begin{quote}
		A utility function is \textbf{randomly generated} if it is drawn from a probability distribution 
		over functions $u : \mathbb{O} \to [0,1]$ that is \emph{symmetric}, \emph{regular}, and \emph{almost surely non-uniform}. 
	\end{quote}
	
	\textcolor{black}{I will use the term \emph{random agent} to refer to an expected utility maximizer with a randomly generated utility function, who chooses uniformly at random between options with equal expected utility. I will also make the standard assumptions that the agents we are interested in are certain that their utilities will not change over time and that they will continue to maximize expected utility at future choice nodes.}
	
	We have said that an expected utility maximizer's final goals are represented by its utility function over final outcomes. Such an agent's \emph{instrumental} goals can be represented, in the context of sequential choice, by its choice dispositions with respect to decision subtrees. At a choice node, an expected utility maximizer decides which path to take by deciding which of the available subtrees offers the opportunity for greatest expected utility. Each subtree represents a possible situation that the agent has the opportunity to put itself in, which it evaluates in terms of the opportunities for utility that it affords.

	We can therefore formalize the idea of instrumental goals in general, and convergent instrumental goals in particular, in terms of relations on the set of possible decision trees $\mathcal{T}_\mathbb{O}$. Broadly speaking, a relation $\atleastasgoodas$ on decision trees represents a convergent instrumental goal if, given a choice between subtrees $t_i$ and $t_j$ where $t_i \atleastasgoodas t_j$, a random agent is likely to choose $t_i$ rather than $t_j$. More precisely, we can define three notions of instrumental convergence corresponding to progressively stronger notions of comparative likelihood.

	\begin{description}
		\item[] A reflexive relation $\succcurlyeq$ on $\mathcal{T}_\mathbb{O}$ is \textbf{weakly instrumentally convergent} iff, given a binary choice between actions $a_i$ and $a_j$ that lead to subtrees $t_i$ and $t_j$ respectively, where $t_i \succcurlyeq t_j$, a random agent (regardless of the particular distribution from which its utility function is generated) is at least as likely to choose $a_i$ as to choose $a_j$.
		
		\item[] An irreflexive relation $\betterthan$ on $\mathcal{T}_\mathbb{O}$ is \textbf{strictly instrumentally convergent} iff, given a binary choice between actions $a_i$ and $a_j$ that lead to subtrees $t_i$ and $t_j$ respectively, where $t_i \betterthan t_j$, a random agent is strictly more likely to choose $a_i$ than to choose $a_j$.
		
		\item[] An irreflexive relation $\betterthan$ on $\mathcal{T}_\mathbb{O}$ is \textbf{absolutely instrumentally convergent} iff, given a binary choice between actions $a_i$ and $a_j$ that lead to subtrees $t_i$ and $t_j$ respectively, where $t_i \betterthan t_j$, a random agent chooses $a_i$ with probability 1.
		
	\end{description}
	
	Informally, a weakly convergent instrumental goal is a ranking of decision trees such that, in a binary choice, any randomly generated expected utility maximizer is at least as likely to choose an action that leads to a higher-ranked subtree. A strictly convergent instrumental goal is a ranking such that any randomly generated expected utility maximizer is \emph{more} likely to choose a higher-ranked subtree. And an absolutely convergent instrumental goal is a ranking such that any randomly generated expected utility maximizer is \emph{almost certain} to choose a higher-ranked subtree. In each case, these claims must hold regardless of the particular distribution from which the utility function is generated, as long as it is symmetric, regular, and almost surely non-uniform. 
	These definitions formalize (with different degrees of strength) the idea that a convergent instrumental goal is something that agents with \emph{most} final goals will have instrumental reason to pursue in any given situation.
	
	\section{Is power a convergent instrumental goal?}
	\label{s:power}
	
	The question we're interested in is whether \emph{power} is a convergent instrumental goal. To answer that question, we now need to formalize power as a relation on decision trees. As we'll see, there's more than one way to do this.
	
	Informally, power in the sense we're interested in is the ability to influence outcomes in the world. A bit more precisely, it's the ability to influence the \emph{probabilities} of outcomes and of events (sets of outcomes). An agent has more power if they can bring about a wider range of probability distributions over outcomes. Different notions of power will correspond to different ways of precisifying the notion of a ``wider range'' of distributions.
	
	Let's start with the simplest sense in which one decision tree can give an agent more power than another, based on set inclusion.

	\begin{description}
		\item[Power 1 (Inclusion)] $t_i \succcurlyeq_{p1} t_j$ iff $\mathcal{L}(t_j) \subseteq \mathcal{L}(t_i)$. $t_i \succ_{p1} t_j$ iff $\mathcal{L}(t_j) \subset \mathcal{L}(t_i)$.
	\end{description}
	
	In English, $t_i$ confers at least as much power as $t_j$ in this first sense if the lotteries available to an agent at the initial node of $t_j$ are a subset of the lotteries available at the initial node of $t_i$; and $t_i$ confers strictly more power if this subset relation is strict. An agent increases its power in this sense when it gains the ability to bring about some particular outcome or set of outcomes more effectively (with greater probability of success) than it could before. The strict relation $\betterthan_{p1}$ is illustrated in Figure \ref{fig:power1inclusion}, where the set of possible lotteries over three outcomes is represented by a triangular region (a ``probability simplex'') and the sets of lotteries available in subtrees $t_i$ and $t_j$ are depicted as subregions.
	
	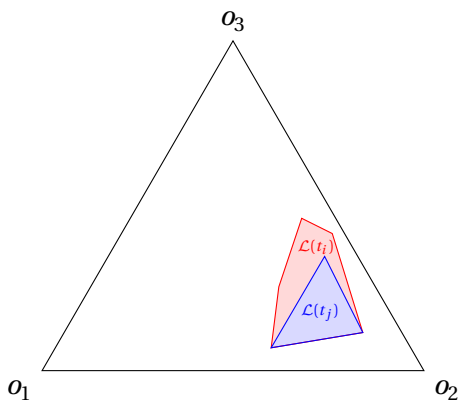
\begin{figure}
		\centering
		\begin{tikzpicture}
			\coordinate (O1) at (0,0);
			\coordinate (O2) at (5,0);
			\coordinate (O3) at (2.5,{5*sin(60)});
			
			\draw (O1) -- (O2) -- (O3) -- cycle;
			
			\node[below left] at (O1) {$o_1$};
			\node[below right] at (O2) {$o_2$};
			\node[above] at (O3) {$o_3$};

			\coordinate (R1) at (4.2,0.5);
			\coordinate (R2) at (3,0.3);
			\coordinate (R3) at (3.7,1.5);

			\coordinate (P1) at (3,0.3);     
			\coordinate (P2) at (4.2,0.5);     
			\coordinate (P3) at (3.8,1.8);     
			\coordinate (P4) at (3.4,2.0);     
			\coordinate (P5) at (3.1,1.1);

			\fill[blue!15] (R1) -- (R2) -- (R3) -- cycle;
			\fill[red!15] (R2) -- (R3) -- (R1) -- (P3) -- (P4) -- (P5) -- cycle;
			
			\draw[red] (P1) -- (P2) -- (P3) -- (P4) -- (P5) -- cycle;

			\draw[blue] (R1) -- (R2) -- (R3) -- cycle;
			
			\node[blue] at (barycentric cs:R1=1,R2=1,R3=1) {\tiny $\mathcal{L}(t_j)$};

			\node[red] (L1label) at (3.58,1.62) {\tiny $\mathcal{L}(t_i)$};
		\end{tikzpicture}
		\caption{An example of the power relation $\betterthan_{p1}$: $\mathcal{L}(t_j)$ (blue) is a proper subset of $\mathcal{L}(t_i)$ (red).}
		\label{fig:power1inclusion}
	\end{figure}
	
	It is easy to show that:
	
	\begin{restatable}[]{proposition}{}\label{t:IC1}
		$\succcurlyeq_{p1}$ is weakly instrumentally convergent. $\betterthan_{p1}$ is strictly but not absolutely instrumentally convergent.
	\end{restatable}
	
	\begin{proof}
		It is trivial that $\succcurlyeq_{p1}$ is weakly instrumentally convergent, and that $\betterthan_{p1}$ is not absolutely instrumentally convergent. The less trivial (though still fairly elementary) fact is that $\betterthan_{p1}$ is strictly instrumentally convergent.

		Suppose that $\mathcal{L}(t_j) \subset \mathcal{L}(t_i)$. Since the number of possible pure strategies in a finite decision tree is finite, and mixed strategies are possible, the set of lotteries $\mathcal{L}(t)$ available at a decision tree $t$ is convex and compact. It then follows from the hyperplane separation theorem \citep[ \S 2.5.1]{boyd2004convex} that, for any lottery $l \in \mathcal{L}(t_i) \setminus \mathcal{L}(t_j)$, there is a utility function $u$ such that $u(l) > u(t_j)$---that is, the expected utility of $l$ is strictly greater than the maximum expected utility of the lotteries in $\mathcal{L}(t_j)$. Moreover, this will remain true under small enough perturbations of $u$---that is, for all $u'$ in some neighborhood of $u$. (Let $u(l) - u(t_j) = \delta$. Then $u'(l) > u'(t_j)$ for any utility function $u'$ that perturbs the $u$ utilities of each outcome by less than $0.5 \delta$.) The assumption of regularity then implies a strictly positive probability that a random agent will have a utility function in this neighborhood of small perturbations of $u$, in which case they will assign strictly greater expected utility to $t_i$ than to $t_j$, and choose action $a_i$. On the other hand, if there is an expected-utility-maximizing lottery in $\mathcal{L}(t_j)$, that same lottery is in $\mathcal{L}(t_i)$, so the agent is equally likely to choose $a_i$ or $a_j$. In combination, these facts imply that a random agent is strictly more likely to choose $a_i$ than to choose $a_j$.
	\end{proof}
	
	This gives us a minimal sense in which power is a convergent instrumental goal: an agent with randomly generated final goals will be more likely than not to choose actions that make a strictly larger set of lotteries available.
	
	The relation $\betterthan_{p1}$ yields strict but not absolute instrumental convergence: a random agent is more than 50\% likely to choose a path that confers greater power in this sense, but is not certain to do so. This reflects the fact that the more powerful subtree will always offer at least as great expected utility, but need not offer strictly \textit{greater} expected utility. A more demanding notion of inclusion, however, gives us a power relation that \emph{guarantees} a strict inequality in expected utility, and hence that an expected utility maximizer will opt for greater power.
	
	\begin{description}
		\item[Power 2 (Interiority)] $t_i \betterthan_{p2} t_j$ iff $\mathcal{L}(t_j)$ is a subset of the interior of $\mathcal{L}(t_i)$.
	\end{description}
	
	The \emph{interior} of a region $R$ in a Euclidean space is the set of points that belong to that region but are not on its boundary. Formally, it's the set of points $x \in R$ such that, for some $\varepsilon$, every point within $\varepsilon$ of $x$ also belongs to $R$. To say that $t_i$ is more powerful than $t_j$ in the sense of $\betterthan_{p2}$, then, means that $t_i$ expands the set of lotteries available in $t_j$ ``in every direction'' (see Figure \ref{fig:power2interiority}). An agent might become more powerful in this sense by gaining access to more of some all-purpose resource (like money, energy, or compute) that increases its probability of success in pursuing not just some particular goals but  \emph{whatever} goals its sets for itself.
	
	\begin{figure}
		\centering
		\begin{tikzpicture}
			\coordinate (O1) at (0,0);
			\coordinate (O2) at (5,0);
			\coordinate (O3) at (2.5,{5*sin(60)});
			
			\draw (O1) -- (O2) -- (O3) -- cycle;
			
			\node[below left] at (O1) {$o_1$};
			\node[below right] at (O2) {$o_2$};
			\node[above] at (O3) {$o_3$};

			\coordinate (R1) at (4.2,0.5);
			\coordinate (R2) at (3,0.3);
			\coordinate (R3) at (3.7,1.5);

			\coordinate (P1) at (2.8,0.2);     
			\coordinate (P2) at (4.3,0.4);     
			\coordinate (P3) at (3.8,1.8);     
			\coordinate (P4) at (3.4,2.0);     
			\coordinate (P5) at (3.1,1.1);

			\fill[red!15] (P1) -- (P2) -- (P3) -- (P4) -- (P5) -- cycle;
			\fill[blue!15] (R1) -- (R2) -- (R3) -- cycle;
			
			\draw[red] (P1) -- (P2) -- (P3) -- (P4) -- (P5) -- cycle;

			\draw[blue] (R1) -- (R2) -- (R3) -- cycle;
			
			\node[blue] at (barycentric cs:R1=1,R2=1,R3=1) {\tiny $\mathcal{L}(t_j)$};

			\node[red] (L1label) at (3.58,1.62) {\tiny $\mathcal{L}(t_i)$};
		\end{tikzpicture}
		\caption{An example of the power relation $\betterthan_{p2}$: $\mathcal{L}(t_j)$ (blue) is confined to the interior of  $\mathcal{L}(t_i)$ (red).}
		\label{fig:power2interiority}
	\end{figure}
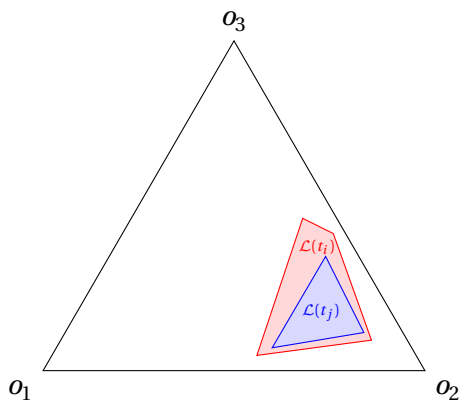
	
	This more demanding notion of power satisfies the most demanding notion of instrumental convergence: 
	
	\begin{restatable}[]{proposition}{}\label{t:IC2}
		$\betterthan_{p2}$ is absolutely instrumentally convergent.
	\end{restatable}
	
	\begin{proof}
		Because we are considering finite decision trees, with only finitely many possible pure strategies, some lottery in $\mathcal{L}(t_j)$ must have maximal expected utility. Let $l \in \mathcal{L}(t_j)$ be such a lottery. If $\mathcal{L}(t_j)$ belongs to the interior of $\mathcal{L}(t_i)$, then for any $l \in \mathcal{L}(t_j)$, there is an $\varepsilon$ such that every point within $\varepsilon$ of $l$ belongs to $\mathcal{L}(t_i)$. Because the randomly generated utility function $u$ is (almost surely) non-uniform, it is (almost surely) possible to increase the expected utility of $l$ by shifting a small amount of probability from a lower-utility to a higher-utility outcome. Specifically, we can find a lottery $l^+$ that is within $\varepsilon$ of $l$, and has strictly greater expected utility. 
		Because $l^+$ within $\varepsilon$ of $l$, it belongs to $\mathcal{L}(t_i)$. Thus, for any non-uniform utility function, we can find a lottery in $\mathcal{L}(t_i)$ with greater expected utility than any lottery in $\mathcal{L}(t_j)$. So a random agent will (almost surely) choose $t_i$ over $t_j$.
	\end{proof}
	
	These first two results 
	might seem too trivial to be interesting. As \cite{gallow2024instrumental} writes, about a slightly different but similarly simple notion of power:
	``[T]here are ways of defining `power' on which it becomes tautological that power would be instrumentally valuable to have. For instance, if we define `power' in terms of the ability to effectively pursue your ends without incurring costs, then it will follow that more power would be instrumentally valuable. Cost-free abilities are never instrumentally disvaluable---just in virtue of the meaning of ‘cost-free’. And the ability to effectively pursue your ends is, of course, instrumentally valuable...Defining ‘power’ in this way makes the convergent instrumental value thesis easy to establish; but for that very reason, it also makes it uninteresting.''\footnote{By ``the ability to effectively pursue your ends'', Gallow means the particular ends that an agent happens to have, so that power in this sense is relativized to the agent's ends/goals/utility function. In that respect, the notion of power he's commenting on is different from $\betterthan_{p1}$ or $\betterthan_{p2}$.}
	
	Propositions \ref{t:IC1} and \ref{t:IC2} aren't tautologies (except perhaps in a sense in which every mathematical truth is a tautology), but they are formally fairly elementary. But that doesn't necessarily make them uninteresting or practically insignificant. There are at least two reasons we might find these initial results noteworthy.
	
	First: Proposition \ref{t:IC1} reflects the fact that a random agent \textit{may} assign strictly greater expected utility to a path in a decision tree that is superior by the $\betterthan_{p1}$ relation, and Proposition \ref{t:IC2} reflects the fact that such an agent \emph{almost surely will} assign strictly greater expected utility to a path that is superior by the $\betterthan_{p2}$ relation. So these results don't just tell us something about the probability that a random agent will make one choice or another when the available options are ranked by these relations. They also tell us that an agent may (resp.\ will) be willing to incur some cost or run some risk in order to put itself in a $\betterthan_{p1}$-superior (resp.\ $\betterthan_{p2}$-superior) situation. Consider the situation depicted in Figure \ref{fig:riskyPowerSeeking}, where the agent must choose between accepting subtree $t_j$ or running a risk (with probability $p$) of a dispreferred outcome $\ooutcome^-$ for a chance (with probability $1 - p$) of ending up in subtree $t_i$. If $t_i \betterthan_{p2} t_j$, then any expected utility maximizer with a non-uniform utility function will choose to go up (running the risk of $\ooutcome^-$ for the sake of ending up in a more powerful position) if $p$ is small enough. And if $t_i \betterthan_{p1} t_j$, an expected utility maximizer may be willing to do so, but would certainly not run any risk for the sake of getting $t_j$ rather than $t_i$.

	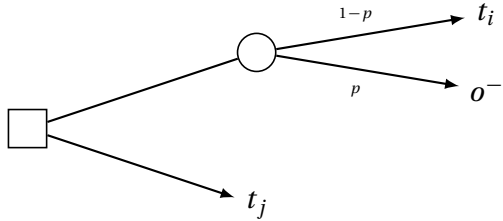
\begin{figure}
		\centering
		
		\begin{tikzpicture}
			
			\node (start) [startstop] {};

			\node[align=center] (C1) [chance, right of=start, xshift=2cm, yshift=1cm] {};
			
			\node[align=center] (D2) [end, right of=start, xshift=2cm, yshift=-1cm] {$t_j$};

			\node (O1-1) [end, right of=C1, xshift=2cm, yshift=.5cm] {$t_i$};
			
			\node (O2-1) [end, right of=C1, xshift=2cm, yshift=-.5cm] {$o^-$};

			\draw	[arrow, thick, black] (start) node[xshift=0.8cm, yshift=-0.5cm] {} -- (D2);
			
			\draw	[thick, black] (start) node[xshift=0.8cm, yshift=-0.5cm] {} -- (C1);

			\draw [arrow, thick, black] (C1) node[xshift=1.3cm, yshift=0.5cm] {\tiny $1 - p$} -- (O1-1);
			
			\draw [arrow, thick, black] (C1) node[xshift=1.3cm, yshift=-0.5cm] {\tiny $p$} -- (O2-1);

		\end{tikzpicture}		
		\caption{A situation in which an agent must choose whether to accept some risk of a dispreferred outcome $\ooutcome^-$ for the chance of ending up in a more powerful position ($t_i$ rather than $t_j$).}
		\label{fig:riskyPowerSeeking}
	\end{figure}
	
	Second: The simple notions of power as inclusion ($\betterthan_{p1}$) and interiority ($\betterthan_{p2}$) plausibly subsume other more particular ends that have been claimed to be instrumentally convergent. And so Propositions \ref{t:IC1} and \ref{t:IC2} provide some vindication of those other more specific instrumental convergence claims. For instance, consider \emph{survival}. If an agent is destroyed, it can have no further effects on the world. The agent's strategy set, post-mortem, is a singleton (``do nothing, with probability 1''), and only a single lottery over outcomes is available to it. With a modicum of idealization, we might suppose that the agent could get this same lottery if it survives, simply by doing nothing. If so, then survival will always be $\betterthan_{p1}$-superior to death, and typically $\betterthan_{p2}$-superior (except when the lottery the agent expects conditional on death happens to be on the precise boundary of the set of lotteries available to it if it survives). So Propositions \ref{t:IC1} and \ref{t:IC2} provide some vindication of the claim that \emph{survival} is a convergent instrumental goal.
	
	Similarly, consider the goal of \emph{acquiring information} (either by gathering new evidence or by cognitive self-improvements that enable an agent to extract more information from its existing stock of evidence). Again with a bit of idealization, the choice whether to acquire new information can be understood as follows: The agent faces some set of $n$ possible actions. If it chooses not to seek further information, it will simply take the action corresponding to the lottery with greatest expected utility, given its initial information. If it does choose to seek out new information, this amounts to going to a chance node, each branch of which leads to a subsequent $n$-action choice, with one action corresponding to each of the $n$ original actions. 
	And the probability-weighted average of the lotteries associated with a given action after acquiring new information is identical to the lottery initially associated with that action. Choosing to acquire free information, then, never reduces the set of available lotteries: Any lottery the agent could have by declining the information, it could still have by adopting the strategy of taking the same action (or probabilistic mixture of actions) regardless of what information it receives. Gaining information, in other words, can only expand the set of available lotteries (at minimum in the sense of $\atleastasgoodas_{p1}$, if not $\betterthan_{p1}$ or $\betterthan_{p2}$). (This is closely related to Good's Theorem \citep{good1966principle}, the well-known result that free information always has non-negative value for an expected utility maximizer.) Propositions \ref{t:IC1} and \ref{t:IC2} thus subsume the claim that information acquisition is a convergent instrumental goal.
	
	Finally, as suggested above, the acquisition of all-purposes resources like money, energy, or compute, insofar as it makes the agent more effective at pursuing \textit{any} ends it sets for itself, can be seen as increasing its power in the sense of $\betterthan_{p2}$. Proposition $\ref{t:IC2}$ could therefore be seen as providing some support for the claim that resource acquisition is a convergent instrumental goal. (This connection is a bit more tenuous, though, since the assumption that generic resources increase the agent's capacity to achieve \emph{any} end involves quite a bit of idealization.)
	
	All that being said, Propositions \ref{t:IC1} and \ref{t:IC2} are undeniably limited in that the relations of inclusion and interiority are very demanding. Most if not all choices involve tradeoffs, with each available action foreclosing some possibilities that other actions leave open. So it will rarely if ever be the case that a real agent is faced with options that are ranked by the relations $\succ_{p1}$ or $\betterthan_{p2}$. This too isn't necessarily fatal to the interestingness of the results. Good's Theorem is still interesting and informative even though information is never or almost never literally free; results about infinitely repeated games are informative even though real-world games are never infinitely repeated. Like free information and infinite games, power in the sense of inclusion or interiority may provide an idealized but still useful approximate description of a class of real-world choice situations. 
	Still, it's natural to wonder whether we can find a more expansive and flexible notion of power than $\betterthan_{p1}$ or $\betterthan_{p2}$ that still represents a convergent instrumental goal.

	Here's one natural thought: 
	How powerful an agent is, in the sense we're interested in, shouldn't depend on \emph{which} particular outcomes or events that agent can bring about with a given probability. The ordinary concept of power is ambiguous on this point---on one natural understanding, \emph{power} is the ability to bring about the outcomes and events \emph{that you desire}. 
	But if we want a notion of power that lets us make predictions about an agent's behavior \emph{in ignorance} of its final goals, then how powerful an agent is shouldn't depend on those goals.  
	We can formalize this idea by once again invoking the notion of permutation invariance: If $\mathcal{L}(t_i)$ is a permutation of $\mathcal{L}(t_j)$, in the sense that $\mathcal{L}(t_i)$ can be obtained from $\mathcal{L}(t_j)$ by applying the same permutation $\pi$ to every outcome in every lottery in $\mathcal{L}(t_j)$, then $t_i$ and $t_j$ confer equal power. 
	
	This suggests a third sense in which one subtree in a sequential choice situation can give an agent more power than another. 
	
	\begin{description}
		\item[Power 3 (Permuted Inclusion)] $t_i \succcurlyeq_{p3} t_j$ iff there is a permutation $\pi$ on $\mathbb{O}$ such that $\pi(\mathcal{L}(t_j)) \subseteq \mathcal{L}(t_i)$, where $\pi(\mathcal{L}(t_j))$ is the set of lotteries obtained by applying permutation $\pi$ to every lottery in $\mathcal{L}(t_j)$. If $\pi(\mathcal{L}(t_j)) \subset \mathcal{L}(t_i)$, then $t_i \succ_{p3} t_j$.
	\end{description}
	
	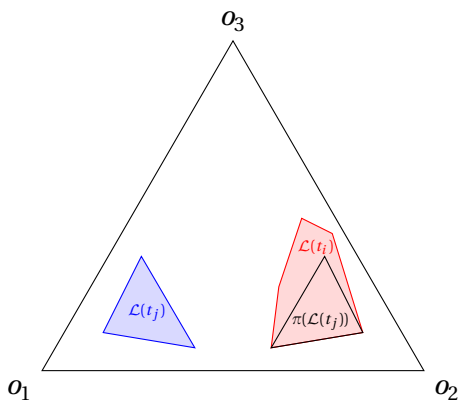
\begin{figure}
		\centering
		\begin{tikzpicture}
			\coordinate (O1) at (0,0);
			\coordinate (O2) at (5,0);
			\coordinate (O3) at (2.5,{5*sin(60)});
			
			\draw (O1) -- (O2) -- (O3) -- cycle;
			
			\node[below left] at (O1) {$o_1$};
			\node[below right] at (O2) {$o_2$};
			\node[above] at (O3) {$o_3$};
			
			\coordinate (L1) at (0.8,0.5);
			\coordinate (L2) at (2,0.3);
			\coordinate (L3) at (1.3,1.5);
			
			\coordinate (R1) at (4.2,0.5);
			\coordinate (R2) at (3,0.3);
			\coordinate (R3) at (3.7,1.5);

			\coordinate (P1) at (3,0.3);     
			\coordinate (P2) at (4.2,0.5);     
			\coordinate (P3) at (3.8,1.8);     
			\coordinate (P4) at (3.4,2.0);     
			\coordinate (P5) at (3.1,1.1);

			\fill[red!15] (P1) -- (P2) -- (P3) -- (P4) -- (P5) -- cycle;
			\fill[blue!15] (L1) -- (L2) -- (L3) -- cycle;
			
			\draw[red] (P1) -- (P2) -- (P3) -- (P4) -- (P5) -- cycle;

			\draw[blue] (L1) -- (L2) -- (L3) -- cycle;
			\draw (R1) -- (R2) -- (R3) -- cycle;
			
			\node[blue] at (barycentric cs:L1=1,L2=1,L3=1) {\tiny $\mathcal{L}(t_j)$};
			\node (pilabel) at (3.64,0.65) {\tiny $\pi(\mathcal{L}(t_j))$};

			\node[red] (L1label) at (3.58,1.62) {\tiny $\mathcal{L}(t_i)$};
		\end{tikzpicture}
		\caption{An example of the power relation $\betterthan_{p3}$: $\mathcal{L}(t_i)$ (red) includes a permutation of  $\mathcal{L}(t_j)$ (blue).}
		\label{fig:power3pi-inclusion}
	\end{figure}
	
	This relation is illustrated in Figure \ref{fig:power3pi-inclusion}. Like the first two power relations, $\betterthan_{p3}$ is a possible precisification of the idea that an agent is more powerful when it has the ability to bring about a wider range of lotteries over outcomes. But now ``wider range'' no longer requires inclusion. Formally, the idea now is that one set of lotteries is larger than another when it has greater measure under any permutation-invariant measure on the whole set of probability distributions over $\mathbb{O}$.
	
	Is power in this sense a convergent instrumental goal? The answer is mixed.
	
	\begin{restatable}[]{proposition}{}\label{t:IC3}
		$\succcurlyeq_{p3}$ is weakly instrumentally convergent. $\betterthan_{p3}$ is neither strictly nor absolutely instrumentally convergent.
	\end{restatable}
	
	\begin{proof}
		The fact that $\succcurlyeq_{p3}$ is weakly instrumentally convergent follows from the fact that the distribution from which a randomly generated utility function is drawn must be permutation-invariant. This implies that, for any set $L$ of lotteries and any permutation $\pi$, it is exactly as likely that some lottery in $L$ has greater expected utility than any lottery in $\pi(L)$ as that some lottery in $\pi(L)$ has greater expected utility than any lottery in $L$. Thus, if $\mathcal{L}(t_i) = \pi(\mathcal{L}(t_j))$ for some permutation $\pi$, a random agent is equally likely to opt for $t_i$ or $t_j$. It's then clear that if $\pi(\mathcal{L}(t_j)) \subset \mathcal{L}(t_i)$, a random agent will be \emph{at least} as likely to opt for $t_i$, since adding lotteries to a set can't reduce the maximum expected utility of the lotteries in that set.
		
		On the other hand, the fact that $\pi(\mathcal{L}(t_j)) \subset \mathcal{L}(t_i)$ need not mean that a randomly generated agent is strictly more likely to opt for $t_i$. To see why, consider the decision tree in Figure \ref{fig:linearSpan1}. Here an agent is faced with two options: Going up immediately results in a lottery with a $\frac{1}{2}$ chance of outcome $o_1$, a $\frac{1}{3}$ chance of $o_2$, and a $\frac{1}{6}$ chance of $o_3$. Going down leads to a further choice between two lotteries, one of which gives a $\frac{1}{3}$ chance of $o_1$, a $\frac{1}{2}$ chance of $o_2$, and a $\frac{1}{6}$ chance of $o_3$, while the other gives a $\frac{1}{6}$ chance of $o_1$, a $\frac{2}{3}$ chance of $o_2$, and a $\frac{1}{6}$ chance of $o_3$. (The agent could also randomize to get any mixture of those two lotteries.) The subtree reached by going down is strictly more powerful than the subtree reached by going up, in the sense of $\succ_{p3}$: Under a permutation that switches $o_1$ and $o_2$, the set of lotteries available in the upper subtree (a singleton) becomes a proper subset of the set of lotteries available in the lower subtree. 
		But a random agent is not more likely to go down: If $u(o_1) > u(o_2)$, the agent will go up; if $u(o_2) > u(o_1)$, the agent will go down; and a randomly generated utility function is equally likely to satisfy either inequality. 
		Together with the assumption of randomization in the case of ties, this means that a random agent is equally likely to go up or down, contra strict instrumental convergence.
		
	\end{proof}
	
	\begin{figure}
		\centering
		
		\begin{tikzpicture}
			
			\node (start) [startstop] {};

			\node[align=center] (C1) [chance, right of=start, xshift=1.9cm, yshift=1cm] {};
			
			\node[align=center] (D2) [decision, right of=start, xshift=1.9cm, yshift=-1cm] {};

			\node (O1-1) [end, right of=C1, xshift=1.9cm, yshift=.5cm] {$o_1$};
			
			\node (O2-1) [end, right of=C1, xshift=1.9cm, yshift=0cm] {$o_2$};
			
			\node (O3-1) [end, right of=C1, xshift=1.9cm, yshift=-.5cm] {$o_3$};
			
			\node (C2) [chance, right of=D2, xshift=1.9cm, yshift=.6cm] {};
			
			\node (C3) [chance, right of=D2, xshift=1.9cm, yshift=-.6cm] {};

			\node (O1-2) [end, right of=C2, xshift=1.9cm, yshift=.4cm] {$o_1$};
			
			\node (O2-2) [end, right of=C2, xshift=1.9cm, yshift=0cm] {$o_2$};
			
			\node (O3-2) [end, right of=C2, xshift=1.9cm, yshift=-.4cm] {$o_3$};
			
			\node (O1-3) [end, right of=C3, xshift=1.9cm, yshift=.4cm] {$o_1$};
			
			\node (O2-3) [end, right of=C3, xshift=1.9cm, yshift=0cm] {$o_2$};
			
			\node (O3-3) [end, right of=C3, xshift=1.9cm, yshift=-.4cm] {$o_3$};

			\draw	[thick, black] (start) node[xshift=0.8cm, yshift=-0.5cm] {} -- (D2);
			
			\draw	[thick, black] (start) node[xshift=0.8cm, yshift=-0.5cm] {} -- (C1);

			\draw [arrow, thick, black] (C1) node[xshift=1.7cm, yshift=0.42cm] {\tiny $1/2$} -- (O1-1);
			
			\draw [arrow, thick, black] (C1) node[xshift=1.7cm, yshift=0.1cm] {\tiny $1/3$} -- (O2-1);
			
			\draw [arrow, thick, black] (C1) node[xshift=1.7cm, yshift=-0.19cm] {\tiny $1/6$} -- (O3-1);
			
			\draw [thick, black] (D2) node[xshift=1cm, yshift=0.25cm] {} -- (C2);
			
			\draw [thick, black] (D2) node[xshift=1cm, yshift=0.25cm] {} -- (C3);

			\draw [arrow, thick, black] (C2) node[xshift=1.75cm, yshift=0.38cm] {\tiny $1/3$} -- (O1-2);
			
			\draw [arrow, thick, black] (C2) node[xshift=1.75cm, yshift=0.1cm] {\tiny $1/2$} -- (O2-2);
			
			\draw [arrow, thick, black] (C2) node[xshift=1.75cm, yshift=-0.15cm] {\tiny $1/6$} -- (O3-2);
			
			\draw [arrow, thick, black] (C3) node[xshift=1.75cm, yshift=0.38cm] {\tiny $1/6$} -- (O1-3);
			
			\draw [arrow, thick, black] (C3) node[xshift=1.75cm, yshift=0.1cm] {\tiny $2/3$} -- (O2-3);
			
			\draw [arrow, thick, black] (C3) node[xshift=1.75cm, yshift=-0.15cm] {\tiny $1/6$} -- (O3-3);
			
		\end{tikzpicture}		
		\caption{A case in which $\betterthan_{p3}$ is not strictly instrumentally convergent.} 
		\label{fig:linearSpan1}
	\end{figure}
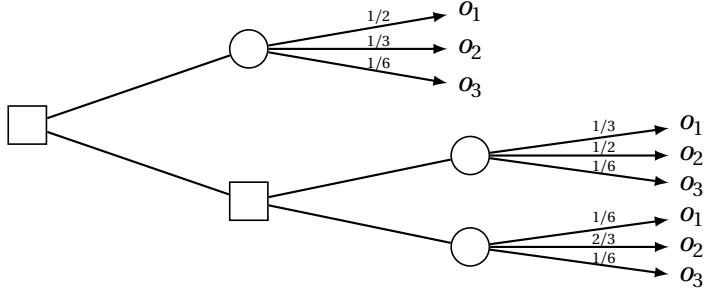
	
	Once again, however, we can guarantee a stronger form of instrumental convergence by replacing simple set inclusion with interiority, as follows:
	
	\begin{description}
		\item[Power 4 (Permuted Interiority)] $t_i \betterthan_{p4} t_j$ iff there is a permutation $\pi$ on $\mathbb{O}$ such that $\pi(\mathcal{L}(t_j))$ is a subset of the interior of $\mathcal{L}(t_i)$.
	\end{description}
	This relation is depicted in Figure \ref{fig:power4pi-interiority}.
	
	\begin{figure}
		\centering
		\begin{tikzpicture}
			\coordinate (O1) at (0,0);
			\coordinate (O2) at (5,0);
			\coordinate (O3) at (2.5,{5*sin(60)});
			
			\draw (O1) -- (O2) -- (O3) -- cycle;
			
			\node[below left] at (O1) {$o_1$};
			\node[below right] at (O2) {$o_2$};
			\node[above] at (O3) {$o_3$};
			
			\coordinate (L1) at (0.8,0.5);
			\coordinate (L2) at (2,0.3);
			\coordinate (L3) at (1.3,1.5);
			
			\coordinate (R1) at (4.2,0.5);
			\coordinate (R2) at (3,0.3);
			\coordinate (R3) at (3.7,1.5);

			\coordinate (P1) at (2.8,0.2);     
			\coordinate (P2) at (4.3,0.4);     
			\coordinate (P3) at (3.8,1.8);     
			\coordinate (P4) at (3.4,2.0);     
			\coordinate (P5) at (3.1,1.1);

			\fill[red!15] (P1) -- (P2) -- (P3) -- (P4) -- (P5) -- cycle;
			\fill[blue!15] (L1) -- (L2) -- (L3) -- cycle;
			
			\draw[red] (P1) -- (P2) -- (P3) -- (P4) -- (P5) -- cycle;

			\draw[blue] (L1) -- (L2) -- (L3) -- cycle;
			\draw (R1) -- (R2) -- (R3) -- cycle;
			
			\node[blue] at (barycentric cs:L1=1,L2=1,L3=1) {\tiny $\mathcal{L}(t_j)$};
			\node (pilabel) at (3.64,0.65) {\tiny $\pi(\mathcal{L}(t_i))$};

			\node[red] (L1label) at (3.58,1.62) {\tiny $\mathcal{L}(t_i)$};
		\end{tikzpicture}
		\caption{An example of the power relation $\betterthan_{p4}$: $\mathcal{L}(t_i)$ (red) includes a permutation of  $\mathcal{L}(t_j)$ (blue) in its interior.}
		\label{fig:power4pi-interiority}
	\end{figure}
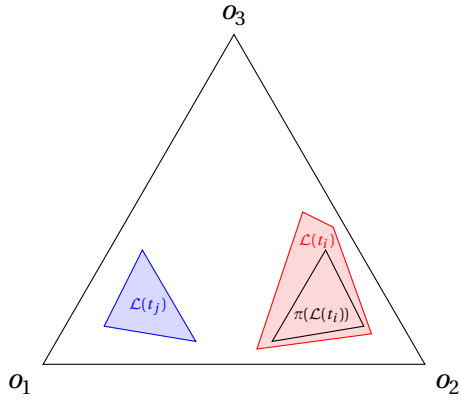
	
	Our final positive result concerns this fourth notion of power:
	
	\begin{restatable}[]{proposition}{ThmPiInteriority}\label{t:IC4}
		$\betterthan_{p4}$ is strictly but not absolutely instrumentally convergent.
	\end{restatable}
	
	\textcolor{black}{The proof of this proposition is longer and less illuminating than the preceding proofs, so I've consigned it to an appendix. The main upshot of Proposition \ref{t:IC4} for our purposes is that we can find a relation more expansive than inclusion, corresponding to the intuitive notion of power as the ability to bring about a wide range of probability distributions over outcomes, that represents a strictly instrumentally convergent goal. We could take this point further, if we wanted to: there are other relations at least somewhat weaker than $\betterthan_{p4}$ that represent strictly convergent goals. But as we will see in the next section, there are limits on how far we can go in ranking decision trees in terms of any notion of power that is strictly or even weakly instrumentally convergent.}

	\section{Power is incomplete}
	\label{s:incompletenessResult}
	
	Taking stock: In the sequential decision-theoretic framework in which we're operating, a natural way of understanding power is that an agent is in a more powerful position when it has the ability to bring about a larger set of lotteries. Different notions of ``larger'' correspond to different ways of ranking situations in terms of power. Of the four power relations we've examined, all are convergent instrumental goals in the weakest sense, one is a convergent instrumental goal in the strongest sense, and three are convergent instrumental goals in the intermediate (``strict'') sense.
	
	\textcolor{black}{A limitation of these results, in terms of their ability to predict the behavior of agents with unknown final goals, is that all of the power relations we have examined are \textit{incomplete}: many pairs of decision trees cannot be ranked in terms of any of the four relations. It's natural to wonder, then, whether we can find a \emph{complete} ranking of decision trees that represents a convergent instrumental goal, in any of the three senses defined in \S \ref{s:instrumentalConvergence}. In particular, is there a numerical measure of the power conferred by any given decision tree such that a random agent is always likely to opt for greater power by this measure---or better yet, such that we can \emph{quantify} the probability of a random agent opting for one subtree rather than another based on their relative power? 
		For example, we might try to measure an agent's power in decision tree $t$ by taking an average over every event $E$ (i.e., every subset of $\mathbb{O}$) of the maximum probability of $E$ across all the lotteries in $\mathcal{L}(t)$---that is, the maximum probability that the agent can give to $E$ by pursuing some strategy in its strategy set. 
		Informally, this measure represents the agent's \textit{average power to bring about any given event}. This would give a complete ranking of decision trees, corresponding at least roughly to the intuitive notion of power. But would it represent a convergent instrumental goal?}
	
	Given the very limited assumptions we have built into the notion of a randomly generated utility function, it turns out that the answer to these questions  is ``no''. Without knowing \emph{something} about an agent's utility functions beyond the minimal assumptions of regularity, symmetry, and non-uniformity, we can't always say which of two subtrees a randomly generated agent is more likely to opt for, let alone how much more likely.
	
	\begin{restatable}[]{proposition}{IncompletenessThm}\label{t:incompleteness}
		There is no total reflexive relation $\atleastasgoodas$ on $\mathcal{T}_\mathbb{O}$ that is weakly instrumentally convergent and whose irreflexive part is strictly instrumentally convergent.
	\end{restatable}
	
	\begin{proof}
		Let $n = |\mathbb{O}|$ denote the cardinality of the set of outcomes. 
		Consider the following decision tree: The agent initially faces two possible actions, $a_i$ and $a_j$, leading to subtrees $t_i$ and $t_j$. Each of those subtrees offers a further $n$-action choice. In $t_i$, the agent can choose one outcome to eliminate, and face a uniform lottery over the remaining outcomes (each having probability $1/(n - 1)$). In $t_j$, the agent can choose to \emph{increase} the probability of any one outcome, and face a lottery that yields that outcome with probability $3/(n+2)$ and every other outcome with probability $1/(n + 2)$. Now suppose, on the one hand, that the agent's utility function is drawn from a distribution that is guaranteed to assign utility 0 to a single outcome and utility 1 to every remaining outcome. Then the agent will prefer $t_i$ (which has an expected utility of 1) over $t_j$ (which has an expected utility of $(n+ 1)/(n+2)$). On the other hand, suppose the agent's utility function is drawn from a distribution that is guaranteed to assign utility 1 to a single outcome and utility 0 to every remaining outcome. Then the agent will prefer $t_j$ (which has an expected utility of $3/(n + 2)$) over $t_i$ (which has an expected utility of $1/(n-1)$).
		
		The distributions described above don't satisfy regularity, but they can be approximated arbitrarily closely be distributions that do, yielding probabilities arbitrarily close to 1 either that a random agent will choose $t_i$ or that a random agent will choose $t_j$.
		
		Any total relation $\atleastasgoodas$ on $\mathcal{T}_\mathbb{O}$ must, by definition, have either $t_i \atleastasgoodas t_j$ or $t_j \atleastasgoodas t_i$. But we have seen that a random agent cannot be guaranteed to be at least as likely to choose $t_i$ as to choose $t_j$, or vice versa, without constraints on the distribution from which its utility function is drawn that go beyond symmetry, regularity, and non-uniformity. So no total relation $\atleastasgoodas$ on $\mathcal{T}_\mathbb{O}$ can be weakly instrumentally convergent. \textcolor{black}{Likewise, any irreflexive relation $\betterthan$ on $\mathcal{T}_\mathbb{O}$ that ranks pairs like $t_i$ and $t_j$ will not be strictly (or absolutely) instrumentally convergent.}
	\end{proof}
	
	The basic observation behind Proposition \ref{t:incompleteness} is that, in some situations, which path a random agent is more likely to follow depends on features of the distribution over utility functions that aren't pinned down by the assumptions of symmetry, regularity, and non-uniformity. This reflects, at least in part, the fact that different sorts of utility functions will exhibit preferences for different and not-easily-comparable sorts of power. An ``easy to please'' agent who finds most outcomes near-maximally good will particularly desire the power to protect itself against a few worst-case outcomes. A ``hard to please'' agent who finds most outcomes near-maximally bad will particularly desire the power to increase the probabilities of the few outcomes it considers highly desirable. Similarly, an agent whose utilities follow a low-variance unimodal distribution, with most utilities tightly clustered and only a few outliers, will want the power to selectively influence the probabilities of those few extreme outcomes, 
	while an agent whose utilities are highly bimodal (e.g., clustered equally near 0 and 1) 
	will mainly want the power to ensure an outcome in the upper half of the distribution, without assigning much special significance to the \textit{very} best or worst outcomes. Thus, predicting which path an agent is more likely to follow in a sequential choice situation can require fairly substantive information about the process or distribution from which its utility function is generated.

	The lesson is that, while instrumental convergence arguments in general and convergent power-seeking in particular may allow us to predict agents' behavior in \emph{some} situations even from a standpoint of extreme ignorance about the agent's final goals, there will be other situations in which those arguments fall silent. 
	\textcolor{black}{And perhaps more to the point, any notion of \emph{power} derived purely from first principles or conceptual analysis, without reference to the particular process by which the utility functions of a particular class of agents are generated, will be either incomplete (telling us that some pairs of situations can't be compared in terms of power) or non-predictive (not constituting a convergent instrumental goal as we've defined it).}

	\section{Absolute power}
	\label{s:absolutepower}
	
	Where does this leave the idea of instrumentally convergent power-seeking? It's a bit of a mixed result. There is certainly some truth to the claim that power is a convergent instrumental goal, on certain natural ways of characterizing power. But this fact \emph{might} turn out to have limited predictive utility, since real-world alternatives will not always be ranked in terms of the relevant power relations. When they are not, considerations of instrumental convergence may not tell us anything about the likely behavior of future agents with radically unknown final goals.
	
	The positive results in \S \ref{s:power} will be more informative, however, about the behavior of agents who have the opportunity to achieve \emph{absolute} or near-absolute power. Absolute power is the ability to bring about any outcome with probability 1. Absolute power in this sense has maximal expected utility for any expected utility maximizer, and is ranked strictly above any other situation an agent could find itself in by the power relations $\succ_{p1}$ and $\succ_{p3}$. (In the relations $\succ_{p2}$ and $\succ_{p4}$, it is ranked above any situation in which there is no event---that is, no set of outcomes---that the agent can bring about with certainty.) Thus, an expected utility maximizer with non-uniform preferences will weakly prefer absolute power to any other situation it could find itself in, strictly prefer it to most alternative situations, and potentially be willing to incur large costs or risks to achieve absolute power if the alternative is a situation in which it has very limited power to influence the course of events or bring about high-utility outcomes.
	
	Of course, like the relations $\betterthan_{p1}$ and $\betterthan_{p2}$, absolute power is an idealization that is very unlikely to be realized in practice---arguably impossible, if outcomes are individuated finely enough, or if empirical certainties are rationally prohibited. But both the expected utility of a decision subtree for a given agent and the probability of a random agent choosing one subtree rather than another are continuous with respect to the sets of lotteries available at the relevant subtrees, so \emph{nearly} absolute power (the ability to bring about \emph{nearly} any outcome with \emph{near}-certainty) will be nearly as attractive to expected utility maximizers as fully absolute power.

	Thus, if you expect that some future AI agents will have superhuman capabilities that give them a good shot at achieving absolute or near-absolute power, then the incompleteness of power relations like $\betterthan_{p1}$--$\betterthan_{p4}$ may not do much to undermine the predictiveness of instrumental convergence. Absolute or near-absolute power will be tempting for most expected utility maximizers, unless the prospects they can expect if they don't seek power are already very desirable, or the costs of power-seeking are very high. On the other hand, suppose you expect future AI agents to be embedded in a complex multipolar world, much like the present-day human world, where even the most powerful individuals cannot aspire to anything like absolute power, and the pursuit of power typically carries great risks and forecloses certain desirable options (like a quiet, carefree, private life). In that case, the fact that power is a convergent instrumental goal may end up telling us relatively little about the likely behavior of those future agents.
	
	\section{Conclusion}
	\label{s:conclusion}

	I will close by noting a few of the important questions this paper has left unaddressed.
	
	First, apart from Proposition \ref{t:IC2}, which establishes that $\betterthan_{p2}$ is absolutely instrumentally convergent, I have not made any quantitative claims about \emph{how much} more likely a random agent is to choose more power over less power. In general, such quantitative claims will require precise assumptions about the distribution from which the agent's utility function is drawn. If we assume, for instance, that the utility of each outcome is drawn independently from a uniform distribution on $[0,1]$, then in any given choice situation, it is possible to calculate the probability that a random agent will make any particular choice or sequence of choices. I will leave these quantitative exercises for future research. In a certain sense, in the context of worrying about catastrophic risks from advanced AI, the precise probabilities are beside the point: a 50/50 chance that the first superintelligent AI agent will exhibit power-seeking tendencies would be more than enough to worry about! What would be most interesting is to calculate the probabilities that such an agent would choose to seek power when it has only a limited chance of success, and must pay some high cost if it fails. 
	The challenge is to develop well-motivated assumptions about the particular probability distribution from which this agent's utility function would be drawn, and about the costs or risks it would have to incur to seek power over humanity.
	
	Second, I haven't considered the possibility that AI agents might have ``non-consequentialist'' preferences regarding their own actions. (I put ``non-consequentialist'' in scare quotes because it's not clear whether these preferences must be non-consequentialist in any strict sense.) Intuitively, an agent might have an \emph{aversion to power-seeking}, and so be unlikely to engage in power-seeking behavior even when it maximizes expected utility. Allowing for such preferences would, I think, prevent us from claiming that any goal is \emph{absolutely} instrumentally convergent: We can't know for certain that an agent will take any particular action in any particular situation if they might have an arbitrarily strong aversion to that particular kind of action. It's not clear, though, that the possibility of non-consequentialist preferences should make any difference to claims of weak or strict instrumental convergence, as long as we maintain the assumption of \emph{ignorance} about an agent's basic preferences. Under that assumption, it's natural to suppose, an intrinsic aversion to power-seeking is no more or less likely than an intrinsic \emph{attraction} to power-seeking. So while instrumental convergence will have less predictive power for agents whose choices are based partly on non-instrumental considerations, such agents might still be on balance \emph{more} likely to choose options that put them in more powerful positions. But fully investigating the implications of non-consequentialist preferences for instrumental convergence arguments would require a paper unto itself.
	
	Third, I haven't said anything about  the relationship between an agent seeking power for itself and seeking to \emph{dis}empower \emph{other} agents. To what extent, and in what circumstances, is power rivalrous or zero-sum, such that the former goal implies the latter? Power as I have characterized it in this paper is not inherently rivalrous. For instance, suppose a group of us are ordering pizza together, that we all happen to prefer mushrooms and green peppers, but that we're all such meek, deferential souls that, if any of us were to suggest a different combination of toppings, the others would go along with it. In this situation, in the sense of power I have been concerned with, \emph{each} of us has absolute power: Each of us \emph{could}, given the actual dispositions of everyone else in the room, bring about any possible outcome with certainty (or, by randomizing our own suggestion, any lottery over outcomes). But it seems plausible that, when agents want different things and are not willing to defer to each other's preferences, power must be at least partly rivalrous: I have the power to get the outcome I want only if you lack the power to prevent it, and vice versa. Under these circumstances, an agent who seeks power will also, at least to some extent, seek to disempower other agents. It would be valuable to investigate these questions formally and identify the circumstances under which disempowerment of other agents is a convergent instrumental goal.
	
	Finally, and most importantly, I should reiterate that my aim in this paper has not been to give a complete defense of the claim that instrumentally convergent power-seeking makes advanced AI agents a catastrophic or existential threat to humanity. In particular, I haven't said anything in defense of the orthogonality thesis or the claim that aligning future AI agents will be very difficult. So I haven't made any argument that we should regard the final goals of actual future AI agents 
	as randomly generated or anything close to randomly generated. 
	For all I've said, it could turn out that these agents will predictably have values that make them disinclined to power-seeking in real-world circumstances, or that are well enough aligned with human values that any power they acquire would be benign.
	All I have tried to establish in this paper is that 
	the abstract thesis that power is a convergent instrumental goal involves a non-trivial grain of truth.\ifanon\else\footnote{For helpful feedback on earlier versions of this paper, I am grateful to Adam Bales, Simon Goldstein, Harry Lloyd, Andreas Mogensen, Bradford Saad, Nate Sharadin, Elliott Thornley, David Thorstad, and participants in the 2024 ``Language, Mind, and Ethics in AI'' workshop at UT Austin.}\fi
	
	\ifanon
	\vspace{\baselineskip}
	\noindent \textit{Competing interests: The author(s) have no relevant financial or non-financial interests to disclose.} 
	\else \fi
	
	\appendix
	
	\section{Appendix: Proof of Proposition 4}
	
	\ThmPiInteriority*

	\begin{proof}
		It is easy to see that $\betterthan_{p4}$ is not absolutely instrumentally convergent. Suppose, for instance, that choosing $a_i$ leads to a further choice between certainty of outcome $o_1$ or certainty of outcome $o_2$, while $a_j$ leads to a 50/50 lottery over outcomes $o_3$ and $o_4$. Then $\pi(\mathcal{L}(t_j))$ is a subset of the interior of $\mathcal{L}(t_i)$ under a permutation $\pi$ that switches $o_1$ with $o_3$ and $o_2$ with $o_4$. But $a_j$ will be chosen by an agent who prefers both $o_3$ and $o_4$ to $o_1$ and $o_2$, and regularity guarantees that utility functions with that feature have non-zero probability.
		
		Next we will show that $\betterthan_{p4}$ is strictly instrumentally convergent. First, if $\mathcal{L}(t_j) \subset \mathcal{L}(t_i)$, then a random agent is more likely to opt for $t_i$ by Proposition \ref{t:IC1}. So we can focus on the case where $\mathcal{L}(t_j) \not \subset \mathcal{L}(t_i)$. Since $\pi(\mathcal{L}(t_j)) \subset \mathcal{L}(t_i)$, this means that $\mathcal{L}(t_j) \not \subseteq \pi(\mathcal{L}(t_j))$. This implies, conversely, that $\pi(\mathcal{L}(t_j)) \not \subseteq \mathcal{L}(t_j)$. (If one set were a proper subset of the other, then given that both sets are closed, one would have greater Lebesgue measure than the other. But permutations are measure-preserving.)
		It then follows from the hyperplane separation theorem that there is a utility function $u_j$ that assigns some lottery in $\mathcal{L}(t_j)$ greater expected utility than any lottery in $\pi(\mathcal{L}(t_j))$, and likewise a utility function $u_{\pi j}$ that assigns some lottery in $\pi (\mathcal{L}(t_j))$ greater expected utility than any lottery in $\mathcal{L}(t_j)$.
		
		Now, consider the set of convex combinations of $u_j$ and $u_{\pi j}$: that is, the set of utility functions $u_\alpha$ such that for any outcome $o \in \mathbb{O}$, $u_\alpha(o) = \alpha u_j(o) + (1 - \alpha)u_{\pi j}(o)$, for some $\alpha \in [0,1]$. Clearly $u_\alpha(t_j) > u_\alpha(\pi(t_j))$ 
		for large enough values of $\alpha$, and $u_\alpha(\pi(t_j)) > u_\alpha(t_j)$
		for small enough values of $\alpha$. Moreover, the continuity of expected utility guarantees that there is an $\alpha^*$ such that $u_{\alpha^*}(t_j) = u_{\alpha^*}(\pi(t_j))$. But because $\pi(\mathcal{L}(t_j))$ belongs to the interior of $\mathcal{L}(t_i)$, any non-uniform $u_{\alpha^*}$ will strictly prefer some $l \in \mathcal{L}(t_i)$ to any $l \in \mathcal{L}(t_j)$: We can slightly improve the most-preferred lottery in $\pi (\mathcal{L}(t_j))$ (moving some probability from the lowest-utility outcome in the support of the lottery to a higher-utility outcome, for instance) while remaining inside $\mathcal{L}(t_i)$. (This depends on $u_{\alpha^*}$ being non-uniform. But given at least three outcomes in $\mathbb{O}$, we can choose utility functions $u_j$ and $u_{\pi j}$ that don't have a uniform function in their convex hull.) Having found an $l \in \mathcal{L}(t_i)$ such that $u_{\alpha^*}(l) > u_{\alpha^*}(t_j)$, this inequality will also hold for small enough perturbations of $u_{\alpha^*}$ (by the same reasoning used in the proof of Proposition \ref{t:IC1}). And this will include some utility functions that strictly prefer $t_j$ to $\pi(t_j)$---for instance, $u_\alpha$ for values of $\alpha$ slightly larger than $\alpha^*$, and any small enough perturbations of those $u_\alpha$.
		
		This means that there is a set of utility functions, with positive Lebesgue measure, that strictly prefer $t_i$ to $t_j$ and $t_j$ to $\pi(t_j)$. From Proposition \ref{t:IC3}, we know that a randomly generated utility function is equally likely to prefer $t_j$ over $\pi(t_j)$ as to prefer $\pi(t_j)$ over $t_j$. The fact that a positive-measure set of utility functions that prefer $t_j$ over $\pi(t_j)$ also prefer $t_i$ over $t_j$ breaks that equality, implying (given regularity) that a randomly generated utility function is strictly more likely to prefer $t_i$ over $t_j$ than the reverse.
		
	\end{proof}
	
	\bibliography{icbib}
\end{document}